\documentclass[letterpaper]{article} %
\usepackage{aaai2027}  %
\nocopyright  %
\usepackage[hyphens]{url}  %
\usepackage{graphicx} %
\usepackage{natbib}  %
\usepackage{caption} %
\usepackage{amsmath,amssymb}   %
\usepackage{tikz}
\usetikzlibrary{arrows.meta, positioning, fit, backgrounds}
\usepackage{booktabs}          %
\usepackage{threeparttable}    %
\usepackage{multirow}          %
\usepackage{float}             %
\newcommand{\ours}[1]{\textbf{#1}}
\newcommand{\snd}[1]{\underline{#1}}
\newcommand{\ci}[2]{{\scriptsize[#1,\,#2]}}

\newcommand{\shell}{\textsc{shell}}
\usepackage{amsthm}
\theoremstyle{plain}\newtheorem{proposition}{Proposition}
\theoremstyle{remark}\newtheorem*{remark}{Remark}
\title{SCORE: Shape-Conforming Regions for Flight in\\
Enclosed, Degraded Environments}
\author{
    Eric Minwoo Kim,
    Jong-Kook Kim
}
\affiliations{
    School of Electrical Engineering, Korea University, Seoul, Republic of Korea\\
    eric2022@korea.ac.kr, jongkook@korea.ac.kr
}

\begin{document}
\maketitle

\begin{abstract}
Autonomous UAVs enter enclosed environments such as caves and collapsed structures that
confine the vehicle and degrade perception. Conformal prediction provides a
distribution-free guarantee by calibrating how far an obstacle keep-out must expand to
absorb perception error at a target coverage level. However, existing keep-out regions use
convex primitives whose bulges consume narrow passages and grow as perception degrades.
Our main contribution defines the nonconformity score on a signed distance field (SDF).
This produces a non-convex keep-out that tightly follows obstacle geometry and avoids the
unnecessary bulging of equal-margin convex regions. Two supporting components keep this
geometry usable as perception degrades. First, a voxelwise union of complementary sensor
observations certifies voxels that any single sensor misses. Second, the margin around the
obstacle adapts to measured visibility without weather labels or the online ground-truth
feedback that single-pass flight cannot provide. Results on real
subterranean data show that the resulting distribution-free, shape-conforming keep-out
retains more usable free space than convex baselines at the same certified coverage, and
produces safer closed-loop flight.
\end{abstract}

\section{Introduction}
Autonomous UAVs are increasingly the only vehicles able to explore enclosed environments such
as caves, mines, and collapsed structures that people cannot safely enter
\cite{tranzatto2022,subtmrs2024}. These environments degrade perception, as collapse raises
dust, fire produces smoke, and caves hold water vapor and darkness
\cite{subtmrs2024,bijelic2020}. Confinement compounds the difficulty, because a narrow passage
offers no wider berth to divert into, and an overly conservative plan sacrifices the mission
rather than a little efficiency. Safe flight in this regime hinges on how perceptual
uncertainty becomes a planning margin.

An autonomy stack plans against a perceived estimate of obstacle geometry, and that estimate
is never exact. The residual gap to ground truth widens as perception degrades
\cite{bijelic2020,perceivewc2024}. A safe margin must therefore be calibrated rather than
merely generous, because too large a margin seals the passage while too small a margin strikes
the wall. Conformal prediction (CP) supplies such a margin, a distribution-free guarantee that
the true obstacle lies within the certified keep-out, and prior work applies it to planning
\cite{vovk2005,angelopoulos2023,lindemann2023,cleaveland2024}. Existing keep-outs inflate
convexly \cite{lindemann2023,tumu2024,shin2026}, and in an enclosed space the obstacle
boundary wraps the free space, and the convex bulge fills the very pocket the vehicle must fly.
As degradation forces the margin to grow and enclosure removes the room to absorb it, the
passage seals and the planner freezes \cite{trautman2010}. The open problem is therefore not
whether to certify, which CP already does, but the shape of the certified region.

This paper makes the shape of the conformal keep-out its central variable. The proposed
keep-out is the level set of a per-voxel signed distance field (SDF)
\cite{oleynikova2017,egoplanner2021}, a non-convex region that hugs concave obstacle geometry
rather than bulging past it, and it is provably contained in every equal-margin convex region
(Proposition~\ref{prop:geom})~\cite{wang2023pcp}. Two further components make this geometry
usable as multi-sensor perception degrades. The per-sensor distance fields are fused per voxel
by a union that certifies cells no single sensor can bound, and the margin is conditioned on
measured visibility without weather labels.

This paper makes three contributions.
\begin{enumerate}
  \item A non-convex SDF-shell conformal keep-out, provably $\subseteq$-tighter than any convex
  region at equal margin, and empirically freeing $1.10\times$ the usable space of the convex
  hull at matched $0.90$ coverage.
  \item A per-voxel conformal fusion rule whose value is coverage rather than tightening,
  because the union certifies cells any single sensor misses and no per-cell alternative
  meaningfully tightens it.
  \item A label-free severity conditioning that sizes the margin from measured visibility,
  matching label-needing adaptive conformal prediction.
\end{enumerate}
Section~\ref{sec:related} reviews related work and Section~\ref{sec:prelim} the problem
formulation. Section~\ref{sec:method} develops the method, and Section~\ref{sec:exp} reports the
offline and closed-loop evaluation.

\section{Related Work}
\label{sec:related}
Conformal prediction for planning wraps obstacles or predicted agent states in
coverage-guaranteed regions \cite{lindemann2023,cleaveland2024}. Existing constructions
typically use convex primitives or templates \cite{tumu2024}, which waste free space around
non-convex obstacles and do not condition their margins on degraded multi-sensor perception.
Concurrent field-level work conformalizes a residual distance field for dynamic-obstacle
prediction \cite{shin2026}, but assumes a ball robot footprint expressly to sidestep the
generally non-convex obstacle geometry that is our subject, and addresses neither multi-sensor
fusion nor degradation.

A second line adapts margin size rather than region shape. Adaptive conformal inference and its
planning and control extensions \cite{gibbs2021,dixit2023,cuqds2025} update conformal thresholds
from per-step feedback, which requires realized errors unavailable during single-pass flight.
Related methods certify learned perception \cite{perceivewc2024} or learn context-aware
nonconformity scores \cite{lcp2025}, but do not isolate shape-conforming keep-out geometry for
degraded multi-sensor sensing. Our margin is instead conditioned feed-forward on an observed
covariate, without weather labels at deployment, and the headline contribution remains the
region shape.

The controlled geometry baselines are therefore an axis-aligned box, an oriented box, and a
convex-hull instantiation of the convex-template class \cite{tumu2024}. Each is calibrated by
the same protocol to the same target coverage, and the comparison therefore isolates region
shape. Methods addressing dynamic-obstacle prediction or learned-perception uncertainty are
positioned conceptually rather than benchmarked head-to-head.

\section{Preliminaries and Problem Formulation}
\label{sec:prelim}
\paragraph{Split conformal prediction.} Given exchangeable nonconformity scores
$R^{(1)},\dots,R^{(k)}$ and a test score $R^{(0)}$, the
$\lceil(k{+}1)(1-\alpha)\rceil$-th smallest calibration score $C$ satisfies
$\Pr(R^{(0)}\le C)\ge 1-\alpha$ \cite{vovk2005}. In planning, $R$ is the signed
clearance error, defined as predicted minus true clearance, and the keep-out
inflates the obstacle by $C$.
\paragraph{Exchangeability and the observed band.} The guarantee is marginal over
calibration and test and requires exchangeability. Two scoping choices make it honest
here. Coverage is claimed only on the observed band (A3), meaning within field of
view, unoccluded, and in range, because CP cannot certify unperceived obstacles, which
is a property of the theorem rather than a fixable gap. Exchangeability is obtained by
calibrating on held-out static scenes that are disjoint by map, and error histograms are
checked by condition.
\paragraph{Problem formulation.} A UAV flies a static, cluttered environment. Perception
produces per-sensor clearance fields $\hat\varphi_s$ against the true field $\varphi$,
and the dangerous error is the one-sided $A_s=\max(0,\hat\varphi_s-\varphi)$, because
estimating clearance too high is the unsafe direction. Degradation is indexed by a measured
severity $\beta=2.996/\mathrm{MOR}$ that non-stationarily scales the $A_s$ distribution
and can blind a sensor on part of the scene. A blinded sensor cannot certify coverage there
at any finite margin, which makes multi-sensor fusion a necessity for validity rather than an
optimization for tightness. The goal is a keep-out that is valid at level $1-\alpha$ on the observed band
and as navigable as possible given that validity, spending the certified margin without
sealing corridors. The per-voxel union that fuses the sensor fields supplies coverage that no
single sensor can.

\section{Proposed Method}
\label{sec:method}
\subsection{Overview}
\begin{figure}[t]\centering
\resizebox{\columnwidth}{!}{%
\begin{tikzpicture}[
  font=\small,
  >={Stealth[length=1.5mm]},
  b/.style={draw=black!55, rounded corners=2pt, align=center, inner sep=3pt},
  mod/.style={b, fill=blue!6, text width=30mm},
  head/.style={b, fill=orange!15, draw=orange!60, text width=32mm},
  io/.style={b, fill=black!5, text width=34mm},
  ar/.style={-{Stealth[length=1.5mm]}, black!60},
  lbl/.style={font=\scriptsize\itshape, text=black!55}
]
\node[io] (in) at (0,0) {Degraded multi-sensor perception (LiDAR + thermal, fog)};
\node[mod] (fus) at (-2.4,-2.2) {\textbf{Per-Voxel Fusion}\\[1pt]{\scriptsize union across sensors, covers cells one sensor misses}};
\node[mod] (cond) at (2.4,-2.2) {\textbf{Severity Conditioning}\\[1pt]{\scriptsize visibility $\to$ margin $\delta$, widens as visibility drops}};
\node[head] (shell) at (0,-4.5) {\textbf{SDF Shell keep-out} {\scriptsize(headline)}\\[1pt]$\{\varphi\le r_{\mathrm{safe}}+\delta\}$\\[1pt]{\scriptsize non-convex, hugs the wall}};
\node[io] (mpc) at (5.7,-4.5) {MPC planner $\to$ safe trajectory};
\draw[ar] (in) -- (fus);
\draw[ar] (in) -- (cond);
\draw[ar] (fus) -- node[lbl,above,sloped]{fused field} (shell);
\draw[ar] (cond) -- node[lbl,above,sloped]{margin $\delta$} (shell);
\draw[ar] (shell) -- (mpc);
\end{tikzpicture}%
}
\caption{The proposed pipeline. Degraded multi-sensor perception feeds a per-voxel conformal
fusion and a label-free severity conditioning, and the resulting field defines the non-convex
SDF shell keep-out that a sampling-based MPC planner threads.}
\label{fig:method}
\end{figure}
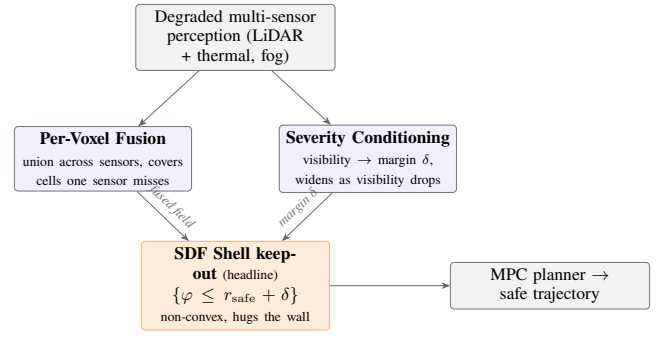
The keep-out is the region around each obstacle that the planner must avoid, sized to keep the
true obstacle inside it. The method builds it in three parts (Figure~\ref{fig:method}). The
first sets its shape, reading the region from the perceived distance to the nearest surface
rather than a convex box or ball that bulges into free space, which lets it hug the obstacle and
leave narrow passages open. The other two keep it reliable as perception weakens. One fuses
several sensors, covering any cell a single sensor misses. The other sets how far it extends,
widening as visibility drops and reading that visibility from an observable signal rather than
weather labels or feedback a single flight cannot supply. Each part is developed in turn below.

\subsection{SDF Shell Geometry}
\label{sec:shell}
\begin{figure}[t]\centering
\includegraphics[width=\columnwidth]{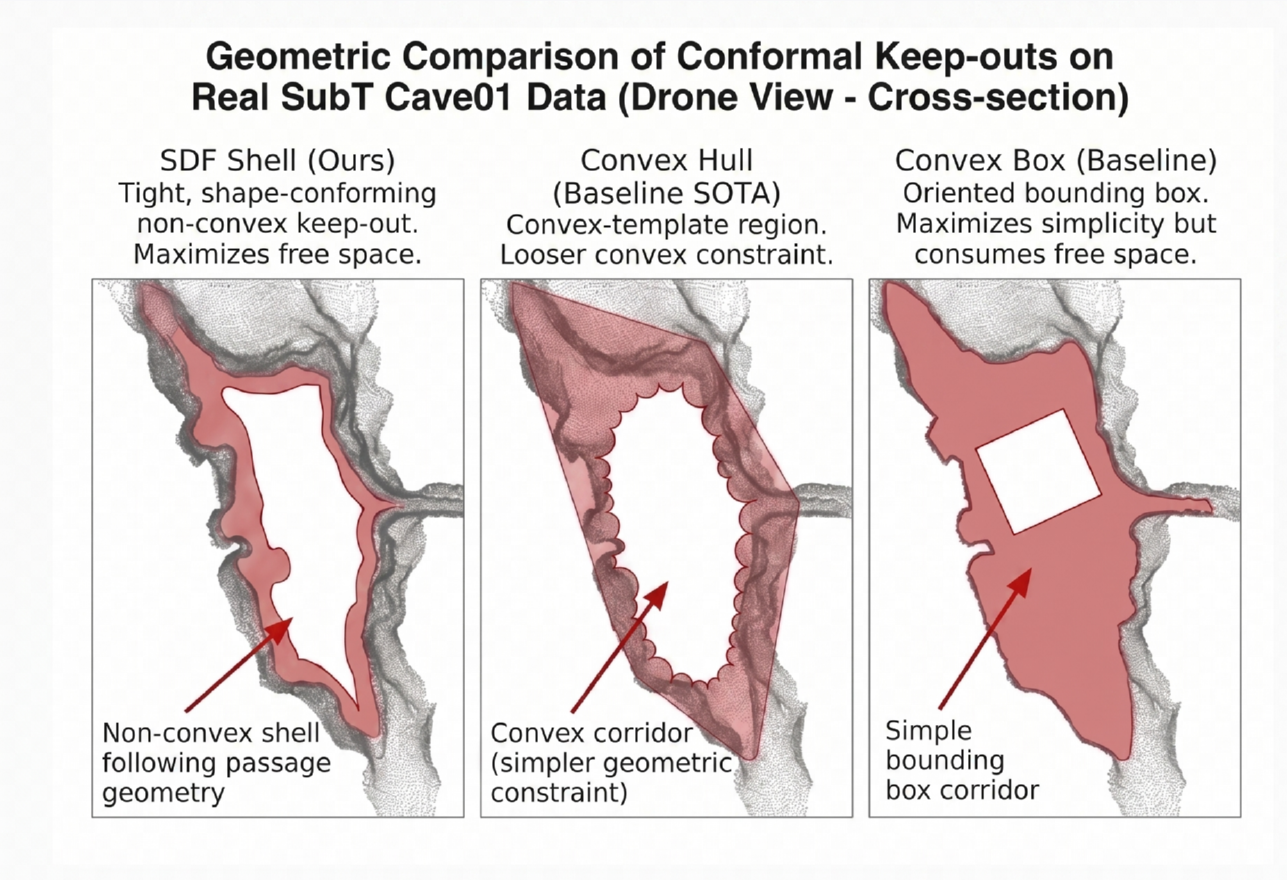}
\caption{Conformal keep-out geometry on a cave passage, a conceptual illustration. The
non-convex shell hugs the wall and leaves the widest corridor, the convex hull bulges inward at
the bends, and the axis-aligned box collapses the corridor to a small rectangle.}
\label{fig:hero}
\end{figure}
The geometry of a conformal region is not chosen directly. It is induced by the
nonconformity score, because for a score $s$ and calibrated quantile $q$ the region is
the level set $\{y : s(y)\le q\}$. The convex keep-outs in common use are exactly the
level sets of convex scores, an $\ell_\infty$, Mahalanobis, or norm score inducing a
box, ellipsoid, or ball. Region
geometry is therefore a \emph{score-design} question, and the conservatism of convex
keep-outs is a property of the scores that induce them.

The score here is defined on the \emph{distance field} rather than on a scalar
per-obstacle margin. Let $\hat\varphi(x)$ be the predicted clearance, that is the
distance to the nearest obstacle surface at a world point $x$, and let $\delta$ be the
calibrated margin from severity conditioning. The induced region, termed the
\emph{shell}, is
\[
  \mathcal S=\{\,x:\hat\varphi(x)\le r_{\mathrm{safe}}+\delta\,\},
\]
the sublevel set of the distance field thickened by $r_{\mathrm{safe}}+\delta$. Because
the score is evaluated against the obstacle field itself, its level set conforms to
concave and elongated geometry instead of bulging past it. The shape is derived from
measured structure rather than learned or assumed, which is what makes the ordering
below deterministic rather than empirical.

The choice of a spatial score is deliberate. A rule reading a single cell's own per-sensor
outputs is capped at the sensor overlap, and escaping that cap requires a score that reads
beyond the cell, which the distance field supplies because $\hat\varphi(x)$ is determined by
obstacle structure around $x$ rather than by one cell's measurement.

\begin{proposition}[Equal-margin $\subseteq$-dominance]\label{prop:geom}
Fix a common obstacle set $O$ and a common margin $t=r_{\mathrm{safe}}+\delta$. The shell
keep-out is the sublevel set $\mathcal S(t)=\{x:\hat\varphi(x)\le t\}=O\oplus B(t)$, and every
convex keep-out is a convex outer approximation of $O$ dilated by the same $t$. Then
\[
  \mathcal S(t)\subseteq\mathrm{hull}(t),\qquad
  \mathcal S(t)\subseteq\mathrm{obb}(t),\qquad
  \mathcal S(t)\subseteq\mathrm{box}(t),
\]
and at equal margin the shell therefore leaves at least as much free space as any convex
keep-out (Figure~\ref{fig:hero}).
\end{proposition}

\begin{proof}
Each convex keep-out is $C(O)\oplus B(t)$ for a convex outer approximation
$C(O)\supseteq O$ with $C\in\{\mathrm{hull},\mathrm{obb},\mathrm{box}\}$. Minkowski dilation
preserves inclusion, that is $A\subseteq A'$ implies $A\oplus B(t)\subseteq A'\oplus B(t)$, and
the shell is exactly $O\oplus B(t)$; dilating $O\subseteq C(O)$ by the common $B(t)$ gives
$\mathcal S(t)\subseteq C(t)$ for each such $C$. Taking complements reverses the inclusions,
hence the shell bounds no more free space than any convex keep-out. The guarantee is deterministic, holds at any convex decomposition of
the obstacle, and needs no distributional assumption.
\end{proof}

\begin{remark}
At a matched coverage target each region is calibrated to its own
margin. The shell hugs the obstacle and carries the least built-in slack, and generally needs a
larger margin than a convex outer approximation to reach the same coverage, hence set inclusion
need not hold there. The free-space magnitude at matched coverage is therefore empirical
(Section~\ref{sec:exp}), where concave savings dominate.
\end{remark}

The non-convex shell enters a
sampling-based model predictive control (MPC) planner as the keep-out constraint. Each sampled rollout is penalized
by its penetration of the conformal level set, and the non-convex constraint is handled
by sequential convex programming that places soft slack on the level set. Because the
margin varies at the voxel level and is severity-conditioned, the keep-out tightens and
loosens according to local geometry and visibility, and the planner can thread the concave
free pocket that convex regions seal. The closed-loop behaviour is reported in
Section~\ref{sec:exp}.

\subsection{Per-Voxel Fusion}
\label{sec:method-fusion}
The distance field is fused at the voxel level. For a sensor $m$ returning cell $c$,
split conformal prediction gives the tightest valid single-sensor keep-out
$u_m(c)=\hat\varphi_m(c)+\delta_m$ from that sensor alone, and a sensor that goes blind
in degradation does not return the cell, which makes abstention automatic. The fused
field combines the returned single-sensor keep-outs by the \emph{union},
$U(x_c)=\max\{u_m(c):c\in R_m\}$, and abstains on cells no sensor returns. The role of
fusion is availability, meaning the bounding of cells that no single sensor can bound,
rather than error averaging. Because a cell returned by a single sensor offers no second value
to combine, a per-cell rule can improve on the union only where two or more sensors return the
same cell, and that overlap is empirically narrow (Section~\ref{sec:exp}).

\noindent The overlap of two or more returns spans only $0.5$ to $5\%$ of the band, and
therefore any per-cell reduction rule can beat the union by at most a few percent. The
value of fusion is therefore availability. Certifying the remaining $95\%$ or more of
the band that no sensor returns requires leaving the per-cell frame for a spatial
information set, which the shell geometry supplies.

\subsection{Severity Conditioning}
\label{sec:cond}
The margin $\delta$ controls the size of the shell. Rather than a single global quantile,
the margin scale is calibrated as a continuous function of measured degradation, using a
normalized conformal score $A/\hat\sigma(\beta)$, and the shell then widens smoothly as
visibility drops without any discrete severity bins. The severity covariate is not read
from an oracle. It is estimated from an observable sensor statistic, the lidar return
count, by an isotonic regression, a monotone non-parametric fit encoding the physical prior
that the margin may only grow as returns are lost, which removes the need for weather labels
at deployment. Conditioning therefore sets how large the margin
is, whereas the shell construction sets how that margin is spent.

Conditioning and the per-voxel union change only the nonconformity
score, and validity is obtained from the split-CP quantile on held-out scores per
condition. This is a rigor backbone, namely standard split conformal applied to a new
geometric construction, rather than a new coverage theorem. The contribution is the
non-convex construction, the frame-optimality result, and the problem they solve.

\section{Experiments}
\label{sec:exp}
\subsection{Overview}
The evaluation proceeds in two stages. The first stage is offline and planner-independent,
establishing on real and simulator data that the shell certifies a tighter region at matched
coverage. The second stage is closed-loop, testing whether that tighter region produces safer
flight in a validated simulation harness. The harness is treated as apparatus, and its
fidelity is established before it is used.

\subsection{Experimental Setup}
\label{sec:setup}
\subsubsection{Common Setup}
\paragraph{Apparatus, datasets, and protocol.} Multi-sensor logs and rollouts come from the
layered simulation harness described below. The data are STF~\cite{bijelic2020}, comprising real adverse weather, $n=8{,}970$,
in five modes, simulator per-cell clearance fields from two sensors (lidar and thermal) at
seven MOR levels, and SubT real subterranean LiDAR for
geometry (an accumulated 14-scan local map, a $5$\,m BEV band at $0.10$\,m,
$r_{\text{safe}}{=}0.30$\,m, and each region inflated to its own reach). Folds are
scene-disjoint. Every arm is calibrated per voxel by split conformal to
exactly $0.90$ realized coverage, and per-mode validity is reported as the coverage gap
$\Delta\mathrm{cov}=\mathrm{cov}-(1-\alpha)$, where $0$ is calibrated and a negative value
falls below nominal.
\paragraph{Modality choice.} The second modality is thermal (LWIR), chosen because long-wave
IR penetrates fog and obscurants roughly $2\times$ better than visible or near-IR~\cite{nebuloni2005},
an advantage that narrows in dense fog, giving a failure mode decorrelated from LiDAR's
obscurant scattering. The method does not rely on a fixed spectral ratio, however. Each
modality's degradation is anchored to real degraded-condition data (STF, FIReStereo), and
therefore the fusion benefit rests on empirically grounded per-modality coverage rather than
an assumed extinction magnitude.

\subsubsection{Offline Evaluation}
\paragraph{Regime caveat.} Rows at MOR\,$\le 2$ are backscatter-artifact-dominated,
because MOR\,$1$--$8$ is $10$--$75\times$ denser than real fog. Physical claims are therefore
capped at MOR-8, and heavier rows are labeled stress tests.

\paragraph{Tightness comparison protocol.} Comparing keep-out tightness between sensor
arms is delicate, because a margin swept to a global target $\alpha$ cannot realize the
same coverage in every arm, which leaves per-frame $\delta$ comparisons
coverage-confounded. The comparison is therefore made per cell, using split conformal
calibrated to exact $1-\alpha$ coverage in each arm by construction, and aggregated
throughout the scene, charging the blind cap to unobservable cells. All tightness
numbers use this matched-coverage-by-construction protocol, and the reported quantities
are the per-cell margin, the fraction of the scene an arm can bound, and the whole-scene
keep-out.

\subsubsection{Simulation Harness}
\label{sec:sim}
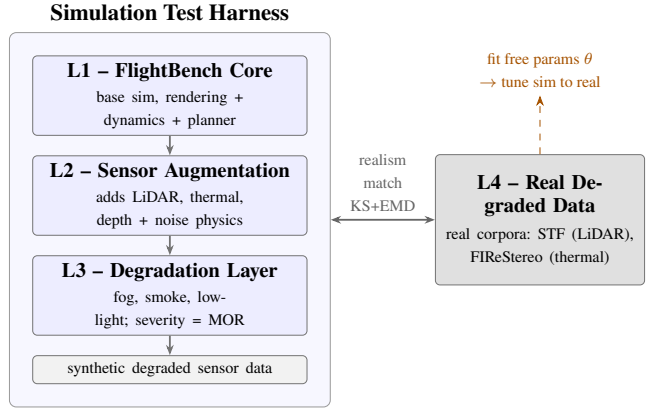
\begin{figure}[t]\centering
\resizebox{\columnwidth}{!}{%
\begin{tikzpicture}[
  font=\small,
  >={Stealth[length=1.6mm]},
  L/.style={draw=black!55, rounded corners=2pt, align=center, inner sep=3pt, text width=40mm, fill=blue!7},
  O/.style={draw=black!55, rounded corners=2pt, align=center, inner sep=3pt, text width=40mm, fill=black!5},
  R/.style={draw=black!55, rounded corners=2pt, align=center, inner sep=3pt, text width=30mm, fill=black!12},
  ar/.style={-{Stealth[length=1.6mm]}, black!55, thick}
]
\node[L] (l1) {\textbf{L1 -- FlightBench Core}\\[1pt]{\scriptsize base sim, rendering + dynamics + planner}};
\node[L, below=3mm of l1] (l2) {\textbf{L2 -- Sensor Augmentation}\\[1pt]{\scriptsize adds LiDAR, thermal, depth + noise physics}};
\node[L, below=3mm of l2] (l3) {\textbf{L3 -- Degradation Layer}\\[1pt]{\scriptsize fog, smoke, low-light; severity = MOR}};
\node[O, below=3mm of l3] (out) {{\scriptsize synthetic degraded sensor data}};
\draw[ar] (l1) -- (l2); \draw[ar] (l2) -- (l3); \draw[ar] (l3) -- (out);
\begin{scope}[on background layer]
\node[draw=black!40, rounded corners=3pt, fill=blue!3, fit=(l1)(l2)(l3)(out), inner sep=3.5mm] (H) {};
\end{scope}
\node[font=\bfseries, above=0.5mm of H.north] {Simulation Test Harness};
\node[R, right=16mm of H.east, anchor=west, minimum height=20mm] (l4)
  {\textbf{L4 -- Real Degraded Data}\\[2pt]{\scriptsize real corpora: STF (LiDAR), FIReStereo (thermal)}};
\draw[<->, thick, black!60] (H.east) -- node[above, align=center, font=\scriptsize]{realism\\match\\ KS+EMD} (l4.west);
\node[align=center, text=orange!65!black, font=\scriptsize, above=9mm of l4] (fit)
  {fit free params $\theta$\\ $\rightarrow$ tune sim to real};
\draw[->, dashed, orange!65!black] (l4.north) to[out=90,in=-90] (fit.south);
\end{tikzpicture}%
}
\caption{The layered simulation harness. A base simulator gains sensor modalities and a
physics-based degradation layer, and the free parameters are tuned to match real adverse-weather
corpora in distribution.}
\label{fig:harness}
\end{figure}
Because no real corpus carries depth, LiDAR, correlated degradation, and a ground-truth map
simultaneously, the multi-sensor and closed-loop evidence is produced on a layered harness. A
simulation-based safety result depends on whether the degradation is real, whether the ground
truth is independent of the sensor, and whether the guarantee holds on the data, and the harness
is validated against all three before any number from it is used.

\paragraph{\textbullet\ Construction.} The harness has four layers (Figure~\ref{fig:harness}). L1 is the closed-loop core, namely
FlightBench~\cite{flightbench2025} providing Flightmare and Unity rendering, Gazebo dynamics,
and ROS on a headless GPU. L2 synthesizes the sensor fields that a game renderer never
computes. L3 injects degradation on the raw per-modality streams before mapping, governed by
the master knob $\beta$. L4 anchors the free parameters of L3 to real
degraded data, as described below.

\paragraph{\textbullet\ Measurement model and ground truth.} At each replan step the logger degrades all
modalities and records, on a shared world-frame voxel grid ($60\times60\times30$ at $0.15$\,m,
that is $9\times9\times4.5$\,m), the per-modality predicted clearance $\hat\varphi_s$, the
ground-truth clearance $\varphi_{gt}$, the per-cell confidence $\hat\sigma_s$, and an observed
mask. The calibrated quantity is the dangerous error $A_s$ from the problem formulation,
evaluated against $\varphi_{gt}$ on the near band ($\varphi_{gt}<0.6$\,m). The ground truth $\varphi_{gt}$ is an exact Euclidean distance transform computed from
the scene mesh, and the render and degradation path never touches it, which makes sensor and
truth independent. Geometry alignment is validated to roughly one voxel (back-projected depth
$2.8$\,cm, LiDAR $2.3$\,cm median), and the injected degradation matches the Beer--Lambert law
$\tau=e^{-\beta z}$ to $98.9\%$ at MOR3.

\paragraph{\textbullet\ Fidelity validation.} The degradation physics is inherited from established models,
namely Koschmieder and Beer--Lambert extinction~\cite{koschmieder1924,middleton1952vision} and
Hahner-style LiDAR fog~\cite{hahner2021fog}, and only the free parameters $\theta^\ast$ are
tuned to real corpora, using STF~\cite{bijelic2020} for LiDAR fog and
FIReStereo~\cite{firestereo2025} for thermal, and the distributional distance is reported
alongside invariant controls. The invariant-control row, on which non-target modalities move by
$\mathrm{KS}\le0.024$ ($p\ge0.82$), proves the tuning is isolated rather than a global
adjustment (Table~\ref{tab:l4}).

\begin{table}[t]\centering\small
\begin{tabular}{@{}llll@{}}
\toprule
Anchor & Real corpus & Result & \\
\midrule
LiDAR fog & STF & KS $0.138$, EMD $0.066$ & \\
thermal & FIReStereo & KS $0.150$, $A\!\downarrow\!40\%$ & \\
invariant & --- & non-target KS $\le0.024$ & \\
\bottomrule
\end{tabular}
\caption{L4 realism anchoring, where the L3 free parameters are tuned to real degraded data and
the invariant row isolates the tuning.}
\label{tab:l4}
\end{table}

Four readiness checks support calibration, detailed in the technical appendix. Separability holds,
the dangerous error $A$ differing by severity strongly enough to justify per-severity conditioning
rather than assuming it. The exchangeability budget is set by measured intra-trajectory
autocorrelation, handled by multi-seed varied-trajectory sweeps, and sparse-severity bins are
flagged provisional. Calibration and test are split by map, leaving no scene shared, and a
leave-one-severity-out check finds held-out coverage conservative but valid at the $0.90$ target on
the available rows, reported as an instrument-consistency check because physical claims are capped
at MOR-8.

\paragraph{\textbullet\ Scope and reproducibility.} The scope of the claim is deliberate. The
claim covers per-failure-mode, finite-sample coverage on the observed local band, conditional
on $\beta$, and includes the construction and its safety consequence. It does not cover global
or full-mission safety, because coverage is band-local. Only the depth camera is a live
sensor, whereas LiDAR and thermal are geometry-synthesized and carry real-anchored degradation
physics, which makes the multi-sensor result a forward-looking simulation-only probe, the only
setting containing all four required ingredients.

Reproducibility rests on the method, hyperparameters, and evaluation protocol specified in the
paper and the technical appendix.

\subsubsection{Closed-Loop Evaluation}
\label{sec:closedloop}
The closed-loop stage isolates the headline region-shape contribution, whereas the fusion and
conditioning supports rest on the planner-independent offline evaluation. The shell is flown against
the box and convex-hull keep-outs~\cite{tumu2024},
using an identical Fast-Planner~\cite{fastplanner2019} on one real cave (SubT
cave01~\cite{subtmrs2024}) in FlightBench~\cite{flightbench2025}. The intermediate oriented box
appears only in the offline geometry ladder. The planner, controller, map,
and start--goal query are held fixed, and every arm is a conformal keep-out matched to a common
certified coverage (approximately $0.93$, per-arm $\ge0.90$) before flight, and outcome
differences therefore isolate region shape at equal safety. An uncalibrated baseline is not
compared, because matched-coverage comparison needs a coverage guarantee it lacks, and
improvement on uncalibrated planning is already established~\cite{lindemann2023,perceivewc2024}.
Perception is stressed by one fog model
along a ladder (MOR $4/12/20/30/40$\,m), and the reported metric is safe-reach, namely a reached
goal (minimum distance below $1.0$\,m) and a non-negative swept wall-clearance. The world goal is a single fixed point for every
arm, placed beyond every keep-out margin, and only the reachability check differs by margin. The
full protocol, thresholds, and configuration provenance are in the technical appendix.
\subsection{Results}
\subsubsection{Offline Evaluation}
The results in this section are distribution-free and planner-independent, computed offline on
real data (SubT and STF) and simulator-synthesized multi-sensor data. They are the foundation
the paper rests on and hold regardless of the closed-loop harness. The experiments obey one
design law, namely that no two effects share a condition. Fusion has two benefits that need
opposite regimes, coverage in degraded conditions where sensors go blind and tightening in the
clear where co-visible sensors decorrelate, and each is therefore tested in its own regime by
the metric that matches the claim, reported per branch alongside the branch's cell fraction.
The headline geometry claim needs neither regime and is tested on its own.

\paragraph{Shell geometry.}\label{sec:exp-geom} The deterministic $\subseteq$-dominance of the shell is measured here as the gain in usable
observed free space on real subterranean geometry at per-region matched $0.90$ coverage. The
shell recovers $1.10\times$ the free space of the convex hull (Tumu-class), $1.11\times$ the
oriented box, and $1.35\times$ the axis-aligned box (Table~\ref{tab:geometry}, $n=80$ held-out
poses). The headline is computed against the
point-cloud ground truth, which is the conservative choice, because swapping to the
reconstructed mesh ground truth moves every ratio in the shell's favour ($1.11\times$,
$1.13\times$, $1.37\times$, gains of $+0.015$ to $+0.025$), as double-wall reconstruction
phantoms tax the shell $4.7\times$ more than the box, a reach relief of $-0.028$ against
$-0.006$\,m. Reporting therefore uses the ground truth biased against the method, and the
full evaluation-audit trail, including the corrections that reduced the reported numbers, is
in the technical appendix.

\begin{table}[t]
\centering
\begin{threeparttable}
\scriptsize
\setlength{\tabcolsep}{2pt}
\begin{tabular}{l cc cc c}
\toprule
& \multicolumn{2}{c}{point-cloud GT} & \multicolumn{2}{c}{mesh GT} & \\
\cmidrule(lr){2-3}\cmidrule(lr){4-5}
Region & free area$\uparrow$ (m$^2$) & \shell/$\cdot\uparrow$ & free area$\uparrow$ & \shell/$\cdot\uparrow$ & $\Delta$reach (m) \\
\midrule
\ours{\shell{} (ours)} & \ours{37.57} & --- & \ours{38.44} & --- & $-$0.028 \\
hull (Tumu-class) & \snd{34.26} & 1.096$\times$ & \snd{34.60} & 1.111$\times$ & $-$0.017 \\
OBB & 33.73 & 1.114$\times$ & 34.06 & 1.129$\times$ & $-$0.014 \\
AABB (box) & 27.87 & 1.348$\times$ & 27.99 & 1.373$\times$ & $-$0.006 \\
\bottomrule
\end{tabular}
\caption{Keep-out region geometry on real subterranean LiDAR (SubT cave01). Best in bold,
second-best underlined. Correction provenance and audit in the technical appendix.}
\label{tab:geometry}
\end{threeparttable}
\end{table}

\FloatBarrier
\paragraph{Per-voxel fusion.}\label{sec:exp-fusion} Two properties of the union,
coverage (availability) and tightness, are tested here in their own regimes. On the fair per-cell grid, where blind cells are charged the cap, the certified fraction (CERT)
of the union exceeds that of the best single sensor in every condition and is CI-separated,
rising from $0.441$ to $0.620$ in the clear and from $0.265$ to $0.373$ at MOR-8
(Table~\ref{tab:fusion-coverage}). The worst column of that table is the minimum CERT of the
physical fog conditions, and the $\kappa$ legs are a thermal-crossover stress test beyond the
measured $\kappa_{\text{real}}{=}0.81$. The worst-case separation between the union and the best
single sensor strengthens from $+0.03$ to $+0.10$ to $+0.17$ as artifact rows are dropped.

Whether fusion also tightens on co-visible cells was tested by sweeping the cross-sensor error
correlation $\rho$. The predicted direction holds, the union tightening at low $\rho$, but net
system tightening is approximately zero, because tightening needs cells that are both co-seen and
decorrelated and that intersection is only $0.4\%$ of the band ($\rho\approx0.37$ to $0.40$ even in
the clear). The honest conclusion is that fusion contributes coverage while tightness comes from
geometry. The full $\rho$-sweep, the alternative fusion rules, and their pre-registered verdicts
are in the technical appendix. An earlier per-frame comparison at unmatched coverage is
coverage-confounded and not used.

\begin{table}[tbp]
\centering
\begin{threeparttable}
\scriptsize
\setlength{\tabcolsep}{1.7pt}
\renewcommand{\arraystretch}{0.95}
\begin{tabular}{l c c ccc cccc}
\toprule
& & & \multicolumn{3}{c}{fog (physical)} & \multicolumn{4}{c}{thermal crossover $\kappa$} \\
\cmidrule(lr){4-6}\cmidrule(lr){7-10}
Arm & \textbf{worst}$\uparrow$ & clear & MOR-8 & MOR-6 & MOR-4 & 0.60 & 0.35 & 0.15 & 0.05 \\
\midrule
\multicolumn{10}{l}{single-sensor}\\
\quad lidar & 0.086 & 0.211 & 0.121 & 0.106 & 0.086 & 0.298 & 0.297 & 0.294 & 0.289 \\
\quad thermal & 0.017 & 0.441 & 0.265 & 0.207 & 0.150 & 0.194 & 0.119 & 0.051 & 0.017 \\
\midrule
\multicolumn{10}{l}{fusion}\\
\quad \ours{union} & \ours{0.186} & \snd{0.620} & \snd{0.373} & \snd{0.303} & \ours{0.231} & \snd{0.460} & \ours{0.396} & \ours{0.336} & \snd{0.303} \\
\quad \shortstack[l]{naive-\\worst} & \snd{0.180} & 0.626 & 0.375 & 0.307 & \snd{0.230} & 0.472 & \ours{0.396} & \snd{0.335} & 0.302 \\
\quad \shortstack[l]{naive-\\Bonferroni\tnote{a}} & 0.187 & \ours{0.651} & \ours{0.392} & \ours{0.320} & 0.240 & \ours{0.499} & 0.426 & 0.366 & \ours{0.332} \\
\bottomrule
\end{tabular}
\begin{tablenotes}\scriptsize
\item[a] Naive-Bonferroni's higher CERT comes from coverage above nominal (cov $\ge 0.976$
against nominal $0.90$) and looser margins, and is therefore not a fair win at matched coverage.
\end{tablenotes}
\end{threeparttable}
\caption{Fusion coverage necessity, the certified fraction (CERT$\uparrow$) of the near-obstacle
band. Best in bold, second-best underlined. Full grid in the technical appendix.}
\label{tab:fusion-coverage}
\end{table}

\paragraph{Severity conditioning.}\label{sec:exp-cond} On real STF the label-free severity
conditioning matches the per-mode labeled-Mondrian ceiling within CI on every mode, at a
maximum coverage deficit of $0.013$ (Table~\ref{tab:conditioning}), and dense fog rises from a
global-margin coverage of $0.762$ to $0.926$ against the $0.941$ labeled ceiling. This is a
parity result relative to adaptive conformal prediction, which attains the same coverage only
given an online per-frame feedback signal unavailable in single-shot flight, a structural
distinction rather than a validity win.

Conditioning the margin on predicted clearance $\hat\varphi$ tightens the mean margin by $20$
to $40\%$ at matched $0.90$ coverage, and it repairs the far-range conditional-coverage floor
that the unconditioned margin lets collapse, raising it from $0.53$ to $0.95$ at MOR-4
(Table~\ref{tab:conditioning}b). The worst bin is the minimum conditional coverage of the
$\hat\varphi$-sextiles, and these gains survive oracle coverage-matching, where all arms are
rescaled to exactly $0.90$ realized (global $0.209/0.213/0.255$ against Mondrian
$0.186/0.166/0.192$) on $1.01\text{M}/0.40\text{M}/0.30\text{M}$ cells. The smooth continuous
form is not deployed, because a
calibration-only transfer certificate refused it after it undercovered a clutter bin on
held-out maps, and the deployed variant is therefore the gate-certified
Mondrian-on-$\hat\varphi$. Every smooth variant that was tested, and its pre-registered accept
rules and verdicts, appears in the technical appendix.

\begin{table*}[tbp]
\centering
\begin{threeparttable}
\scriptsize
\renewcommand{\arraystretch}{0.95}
\setlength{\tabcolsep}{3pt}
\begin{tabular}{@{}l@{\hspace{0.8em}}l@{}}
\textbf{(a) Severity conditioning} & \textbf{(b) Range conditioning}\\[2pt]
\begin{tabular}[t]{l c cc cc cc c}
\toprule
& & \multicolumn{2}{c}{global} & \multicolumn{2}{c}{Mondrian (lab.)} & \multicolumn{2}{c}{cont.\ $\sigma(\hat\beta)$, l-free} & \\
\cmidrule(lr){3-4}\cmidrule(lr){5-6}\cmidrule(lr){7-8}
Mode & $n$ & $\Delta$cov & $\bar\delta$ & $\Delta$cov & $\bar\delta$ & $\Delta$cov & $\bar\delta$ & cont$-$Mond.\ \ci{2.5}{97.5} \\
\midrule
clear & 624 & $+0.007$ & 2.40 & $-0.003$ & 2.39 & $+0.001$ & 2.48 & $+0.003$ \ci{-0.026}{+0.026} \\
light fog & 354 & $-0.120$ & 2.40 & $-0.069$ & 2.50 & $-0.075$ & 2.67 & $-0.007$ \ci{-0.037}{+0.022} \\
dense fog & 202 & $-0.138$ & 2.40 & $+0.041$ & 3.64 & $+0.026$ & 3.24 & $-0.013$ \ci{-0.026}{0.000} \\
snow & 530 & $-0.017$ & 2.40 & $-0.015$ & 2.41 & $-0.008$ & 2.49 & $+0.006$ \ci{-0.031}{+0.032} \\
rain & 308 & $-0.004$ & 2.40 & $-0.069$ & 2.31 & $-0.069$ & 2.35 & $+0.001$ \ci{-0.061}{+0.079} \\
\bottomrule
\end{tabular}
&
\begin{tabular}[t]{l cc cc c}
\toprule
& \multicolumn{2}{c}{global $\delta$} & \multicolumn{2}{c}{Mondrian-$\hat\varphi$} & \\
\cmidrule(lr){2-3}\cmidrule(lr){4-5}
Condition & $\bar\delta$ & worst$\uparrow$ & $\bar\delta$ & worst$\uparrow$ & tighten\ \ci{2.5}{97.5} \\
\midrule
clear & 0.232 & 0.84 & 0.186 & 0.84 & $+0.046$ \ci{+0.038}{+0.055} \\
MOR-8 & 0.232 & 0.81 & 0.167 & 0.80 & $+0.066$ \ci{+0.059}{+0.074} \\
MOR-4 & 0.322 & 0.53 & 0.192 & 0.95 & $+0.130$ \ci{+0.123}{+0.137} \\
\bottomrule
\end{tabular}
\end{tabular}
\begin{tablenotes}\scriptsize
\item $\bar\delta$ is the mean calibrated one-sided safety margin (m), the conformal buffer on
the dangerous clearance error $A=\max(0,\hat\varphi-\varphi)$ at level $1-\alpha$. Lower is
tighter at equal validity. Bracketed quantities are bootstrap $95\%$ confidence intervals.
\item Panels (a) and (b) attach the guarantee to different units, scene-level worst-case in (a)
and per-voxel in (b), on different datasets, and therefore $\bar\delta$ is not comparable across
panels.
\end{tablenotes}
\end{threeparttable}
\caption{Conditioning the conformal margin, on severity for validity in panel (a) and on range
for tightness in panel (b). Full detail in the technical appendix.}
\label{tab:conditioning}
\end{table*}

\subsubsection{Closed-Loop Evaluation}
\begin{figure}[t]\centering
\includegraphics[width=\columnwidth]{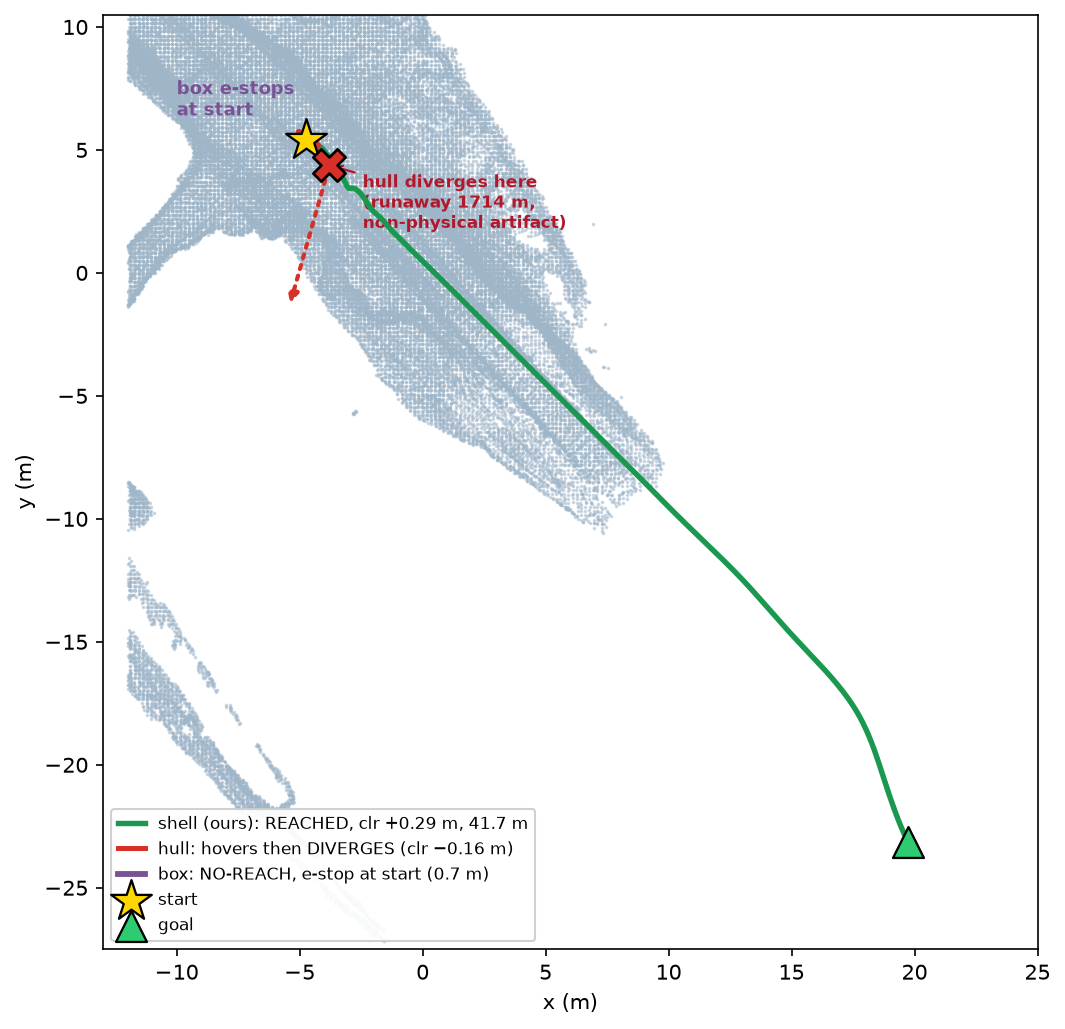}
\caption{Closed-loop flight on the real cave at seed F02 and MOR 4. Only the shell reaches the
goal safely, the hull hovers then diverges, and the box emergency-stops at the start.}
\label{fig:traj}
\end{figure}
All four axes are scored on the reconstructed cloud rather than an arm's own keep-out, namely
safety as collision, reachability as safe-reach (goal reached and no collision), efficiency as
detour, and smoothness as a jerk proxy. Raw reach is
reported only to motivate the correction, namely that the simulator odometry does not halt at
walls and a convex keep-out can otherwise reach the goal after penetrating an obstacle. The
wall-phasing correction removes $7$ raw hull reaches ($42\to35$) and none for the shell
($49\to49$) or the box ($27\to27$).

The controlled contrast is the shell against the box (Figure~\ref{fig:traj}), two stable region operators on the shared
base occupancy, and the separation is decisive, namely safe-reach $49/50=0.98$ (CI $[0.90,1.00]$)
for the shell against $27/50=0.54$ (CI $[0.40,0.67]$) for the box (two-proportion $z=5.15$,
$p<10^{-4}$, Table~\ref{tab:closedloop}). The convex hull sits between at $35/50=0.70$
(CI $[0.56,0.81]$), separated from the shell ($z=3.82$, $p=10^{-4}$) and not from the box
($z=1.65$, $p=0.10$). The shell-against-hull margin has its distribution-free home in the static
geometry ladder, where the shell frees more maneuverable area at matched coverage, and
the closed-loop run corroborates rather than establishes it. Safe-reach is monotone
shell$\ge$hull$\ge$box at every fog severity (Table~\ref{tab:closedloop-grid}).

On this pipeline the shell is uniquely collision-free ($0$ collisions, $0$ divergences) and its
minimum swept clearance never goes negative ($0.285\pm0.021$\,m). The box records $2$ collisions
and the hull $14$ wall-contacts, and both loose regions diverge, the hull on $8$ flights and the
box on $2$. These ten divergences share one deterministic signature. The inflated convex geometry
constricts the exit corridor about $1.5$\,m from the start, and sustained infeasible replanning
there drives Fast-Planner's stock optimizer to a non-physical trajectory, all ten within a narrow
$1.4$ to $1.6$\,m band. The $219$\,m runaway at up to $66$\,m/s, a path far longer than the
roughly $40$\,m cave, is the simulator's ground-truth odometry failing to halt at walls rather
than a physical flight. Shell keeps the corridor open and departs cleanly, never entering this
regime. The shell is also the cheapest region to realize ($0.34$\,ms per map update and flat as fog
thickens, $12\times$ below the hull's $4.01$\,ms, which climbs from $2.97$ to $4.88$\,ms as fog
reveals more geometry and reaches $34.6$\,ms at worst, breaching a $30$\,Hz budget) and calls
the fewest emergency stops ($0.12$ per flight against the box's $1.84$). The box's $0.54$
and hull's $0.70$ against the shell's $0.98$ hold at a configuration
identical for every arm, evidence the outcome tracks region shape rather than tuning. The run is a controlled evaluation at documented configuration (technical appendix),
not a pre-registered confirmatory test, and it differs from the retired preliminary run. The
closed-loop ordering reproduces the static geometry ladder in autonomous flight.

\begin{table}[t]
\centering
\begin{threeparttable}
\scriptsize
\setlength{\tabcolsep}{3pt}
\begin{tabular*}{\columnwidth}{@{\extracolsep{\fill}}l c c c c c@{}}
\toprule
Region & safe-reach\tnote{a} & clearance (m)\tnote{b} & collisions\tnote{c} & e-stops\tnote{d} & compute\tnote{e} \\
\midrule
\ours{\shell{} (ours)} & \ours{49/50 (0.98)} & \ours{0.285\,$\pm$\,0.021} & \ours{0} & \ours{0.12} & \ours{0.34\,ms} \\
convex hull           & 35/50 (0.70)        & 0.208\,$\pm$\,0.165        & 14        & 0.26        & 4.01\,ms       \\
AABB (box)            & 27/50 (0.54)        & 0.280\,$\pm$\,0.046        & 2         & 1.84        & 0.39\,ms       \\
\bottomrule
\end{tabular*}
\begin{tablenotes}\scriptsize
\item[a] Reached goal and non-negative swept wall-clearance.
\item[b] Mean and sd of per-flight minimum clearance of reached flights.
\item[c] Flights recording any negative swept clearance, namely wall contact.
\item[d] Mean emergency-stop recoveries per flight, cap $K{=}3$.
\item[e] Mean per-map-update region-realization time.
\end{tablenotes}
\end{threeparttable}
\caption{Closed-loop region-geometry sweep on SubT cave01 at matched $0.90$ coverage. Best in
bold. Full protocol in the technical appendix.}
\label{tab:closedloop}
\end{table}

\begin{table}[t]
\centering
\begin{threeparttable}
\scriptsize
\setlength{\tabcolsep}{3pt}
\begin{tabular*}{\columnwidth}{@{\extracolsep{\fill}}l c c c c c c@{}}
\toprule
Region & MOR4 & MOR12 & MOR20 & MOR30 & MOR40 & total \\
\midrule
\ours{\shell{} (ours)} & \ours{9} & \ours{10} & \ours{10} & \ours{10} & \ours{10} & \ours{49/50} \\
convex hull           & 8        & 7         & 7         & 7         & 6         & 35/50       \\
AABB (box)            & 6        & 6         & 5         & 5         & 5         & 27/50       \\
\bottomrule
\end{tabular*}
\caption{Safe-reach by fog severity, out of 10 seeds each. Best in bold.}
\label{tab:closedloop-grid}
\end{threeparttable}
\end{table}

\section{Limitations}
\label{sec:limitations}
Coverage is scoped to the observed band, and global-mission safety requires
conservative unknown-space handling. Multi-sensor evidence is semi-synthetic (one
live sensor) and, for the fusion-necessity result, currently a Tier-A corruption
model pending a Tier-B render. The closed-loop result rests on one real cave (SubT cave01) and one planner (Fast-Planner
topological replanning). The hull divergences are geometry-caused rather than an implementation
artifact (Section~\ref{sec:exp}), though the exact optimizer internal that fails during sustained
infeasibility is stock and undocumented, and the corridor-constriction trigger is therefore
claimed confidently while the internal numerics are not. A remaining caveat is that the box and
hull keep-outs share a tile operator that the shell realizes instead by the native signed-distance
path, a possible realization confound.
In-loop coverage sits mildly below nominal ($\approx0.87$--$0.92$).

\section{Conclusion}
\label{sec:conclusion}
Convex conformal keep-outs waste maneuverable space in non-convex clutter, and the
waste compounds as perception degrades and the certified margin grows. This paper contributes a
non-convex signed-distance-field (SDF) shell keep-out, provably $\subseteq$-tighter than any
convex region at equal margin (Proposition~\ref{prop:geom}) and tractable for model predictive
control (MPC), and the margin is spent along obstacle geometry rather than into the corridor.
The margin is fused by a per-voxel union and sized by label-free severity conditioning.
On real subterranean and adverse-weather data the shell tightens the keep-out around obstacles,
freeing about $10\%$ more maneuverable space than the convex hull, the tightest convex baseline,
and $35\%$ more than the box. In closed-loop flight it reaches goals more safely than either.

\clearpage
\bibliography{references}

\end{document}